\documentclass[twoside]{article}

\usepackage[accepted]{aistats2025}
\usepackage[round]{natbib}

\usepackage[utf8]{inputenc} 
\usepackage[T1]{fontenc}    
\usepackage[hidelinks]{hyperref}       
\usepackage{url}            
\usepackage{booktabs}       
\usepackage{amsfonts,bbm,amsmath,amsthm}   
\usepackage{nicefrac}       
\usepackage{microtype}      
\usepackage{xcolor}         
\usepackage[pdftex]{graphicx}
\usepackage{algorithm}
\usepackage{algorithmic}
\usepackage{dsfont}

\theoremstyle{definition}
\newtheorem{definition}{Definition}[section]
\newtheorem{lemma}{Lemma}
\newtheorem{theorem}{Theorem}
\newtheorem{remark}{Remark}

\newtheorem{proposition}{Proposition}

\newcommand{\mbbP}{\mathbb{P}}
\newcommand{\mbbQ}{\mathbb{Q}}

\newcommand{\R}{\mathbb{R}}

\newcommand{\CS}{\textbf{CS}}
\newcommand{\PD}{\textbf{PD}}
\newcommand{\CSPD}{\textbf{CSPD}}
\newcommand{\gCSPD}{\textbf{g-CSPD}}

\begin{document}

\twocolumn[

\aistatstitle{Conformal Prediction Under Generalized Covariate Shift with Posterior Drift}

\aistatsauthor{ Baozhen Wang \And Xingye Qiao}

\aistatsaddress{Binghamton University \And Binghamton University} ]

\begin{abstract}
  In many real applications of statistical learning, collecting sufficiently many training data is often expensive, time-consuming, or even unrealistic. In this case, a transfer learning approach, which aims to leverage knowledge from a related source domain to improve the learning performance in the target domain, is more beneficial. There have been many transfer learning methods developed under various distributional assumptions. In this article, we study a particular type of classification problem, called conformal prediction, under a new distributional assumption for transfer learning. Classifiers under the conformal prediction framework predict a set of plausible labels instead of one single label for each data instance, affording a more cautious and safer decision. We consider a generalization of the \textit{covariate shift with posterior drift} setting for transfer learning. Under this setting, we propose a weighted conformal classifier that leverages both the source and target samples, with a coverage guarantee in the target domain. Theoretical studies demonstrate favorable asymptotic properties. Numerical studies further illustrate the usefulness of the proposed method.
\end{abstract}

\section{INTRODUCTION}
\label{introduction}

Machine learning has achieved great success in many applications, but still has limitations in practice. Ideally, there should be abundant labeled training data that share the same distribution as the test data. However, collecting sufficient training data is often expensive, time-consuming, or even unrealistic. Though semi-supervised learning can alleviate the reliance on labeled training data, in many cases, even unlabeled data are difficult to collect. 

Humans are capable of transferring knowledge across domains \citep{swarup2006cross}. For example, those who have learned the piano can learn the violin faster than others; those who speak a language can learn its dialect faster than those who don't. Inspired by humans' capabilities to transfer knowledge, transfer learning \citep{pan2009,weiss2016}, a new set of principles and methods, aims to leverage knowledge from a source domain to improve the learning performance in a target domain and to lower the reliance on the number of data in the target domain. Transfer learning has been successfully applied in various scenarios, such as texts \citep{wang2011}, images \citep{duan2012learning, kulis2011you, zhu2011heterogeneous}, music \citep{choi2017transfer}, climates \citep{ma2015transfer}, disease predictions \citep{ogoe2015knowledge}, biological systems \citep{zou2015transfer}, and linguistics \citep{prettenhofer2010cross}.

Transfer learning is also known as domain adaptation in the classification setting. Consider labelled data drawn from a source distribution $P$ and a relatively small quantity of labeled or unlabeled data from a target distribution $Q$. Several types of assumptions on how $P$ and $Q$ differ have been studied, such as covariate shift ($\CS$) and posterior drift ($\PD$). $\CS$ \citep{shimodaira2000improving, sugiyama2007direct, kpotufe2021marginal} occurs when the marginal distribution of the features (covariates) in the source data differs from that in the target data, but the conditional distributions of the label given the features (that is, the posterior class probabilities) remain the same between the source and the target data. $\PD$ \citep{cai2021transfer}, on the other hand, occurs when the conditional distributions of the label given the features differ between the source and the target, but the marginal distributions of the features are the same. Special cases of the $\PD$ setting include the real concept drift \citep{gao2007general, gama2014survey}, and the label corruption \citep{van2017theory}. \citet{scott2019generalized} combined the $\CS$ and $\PD$ assumptions and dubbed it \textit{covariate shift with posterior drift} (\CSPD), where the marginal distribution of features differs, just like in $\CS$, and so are the posterior probabilities, just like in $\PD$. See Section \ref{subsec:CSPD} for more details about the distributional settings of transfer learning.

When the distributions of the training and test data differ, traditional statistical and machine learning methods will suffer from misspecification and poor predictive performance. This calls for a set of algorithms and models that are trustworthy, reliable, and conscious about the possible change in data distributions \citep{ben2013robust, namkoong2016stochastic}. Another aspect of trustworthy models are uncertainty quantification. To this end, conformal prediction \citep{vovk2005algorithmic} provides a principled framework in which a prediction set is obtained for each data instance that, with a predetermined probability, covers the true response value  (this probability is also known as the coverage rate) \citep{shafer2008tutorial, vovk2009line, vovk2013transductive, burnaev2014efficiency, lei2014distribution, lei2018distribution}. Such a prediction set may be an interval for regression problems and a set of class labels for classification problems. Conformal prediction is particularly useful in some high-stake application domains in which a misclassified instance can lead to detrimental consequences (e.g., medical diagnosis and national security). By allowing a set-valued prediction for a ``difficult'' instance, one can defer the final decision to a human expert, to a secondary fine-tuned model that is based on more training data, or to sometime later when a more thorough investigation can be conducted.

Conformal prediction relies on the exchangeability assumption to ensure the coverage rate for a prediction set \citep{barber2023conformal}. The exchangeability assumption presumes the sequence of data points can be permuted without altering their joint distribution. However, when $P$ and $Q$ differ, future testing data follow a different distribution from that of the training data, violating the exchangeability assumption. \citet{tibshirani2019conformal} extended the conformal prediction methodology beyond the scenario of exchangeable data, using a weighted version of conformal prediction under the $\CS$ setting. See \citet{lei2021conformal}, \citet{fannjiang2022conformal}, \citet{barber2023conformal}, \citet{cauchois2024robust}, and \citet{wang2025conformal} for developments in conformal prediction methodologies tailored for several variants of the $\CS$ setting. Recent work has explored robustness in conformal prediction under distribution shifts. \cite{ai2024not} propose a fine-grained approach that reweights samples for covariate shift and adjusts confidence levels under a worst-case bound on $f$-divergence for posterior drift, while our framework generalizes $\CSPD$ by relaxing monotonicity conditions. \cite{liu2024multi} propose a multi-source conformal inference framework that reweights data from multiple biased sources, whereas our work focuses on a single-target setting under a more general shift model. Despite these advancements, to our best knowledge, conformal prediction has not been studied in a more general setting such as the $\CSPD$. 

\textbf{Contributions.} We propose a \textit{practical} algorithm of weighted conformal classification that returns a prediction set with a desired level of coverage under the $\CSPD$ assumption. Furthermore, we introduce a less stringent version of $\CSPD$, thereby expanding the applicability of conformal prediction to more general scenarios. One key difference between our work from that of \citet{tibshirani2019conformal} is that, in addition to data from the source domain, our approach also makes use of data in the target domain in the training process. To overcome the non-trivial computational challenge, we exploited Newton's identities and designed a \textit{practical} algorithm to compute the weights. We theoretically show that the proposed method achieves favorable asymptotic properties.


\section{BACKGROUND}\label{sec:review}
\subsection{Notations and Settings}
For $n$ objects $a_1,...,a_n,$ we write $a_{1:n} = \left\{ a_1, ..., a_n \right\}$ to denote their collection. Consider a multi-category classification problem. Let $Z_i = (X_i,Y_i)\in \mathbb{R}^d\times \left\{1,2,\dots,K\right\}, i=1,\dots, m+n$ denote the $i$th instance in the training data set and $Z_T = (X_T,Y_T)\in \mathbb{R}^d\times \left\{1,2,\dots,K\right\}$ denote a test data instance. The training data include $m$ i.i.d. instances from the source domain, $Z_{1:m}$, and $n$ i.i.d. instances from the target domain, $Z_{(m+1):(m+n)}$, while the test data instance $Z_T$ is from the target domain only. Let $P$ denote the distribution of a single data instance $Z=(X,Y)$ from the source data, referred to as the source distribution. For the source data, denote by $P_j$ the conditional distribution of $X$ given $Y=j$ where $j\in\left\{1,2,\dots,K\right\}$, by $\pi_{P,j}=P(Y=j)$ the prior (marginal) probability that the instance $(X,Y)$ belongs to class $j$, and by $\eta_{P,j}(x)=P(Y=j|X=x)$ the posterior (conditional) probability that $(X,Y)$ belongs to class $j$ given $X=x$. Finally denote by $P_X$ the marginal distribution of $X$ for the source data. Analogously we denote the distribution of a single data instance from $Z_{(m+1):(m+n)}\cup Z_T$ as the target distribution $Q$, the corresponding conditional feature distribution as $Q_j$, the prior class probability as $\pi_{Q,j}$, the posterior class probability as $\eta_{Q,j}(x)$ and the marginal feature distribution as $Q_X$ respectively.

\subsection{Weighted Conformal Prediction under Covariate Shift}\label{subsec:wconformal}

Conformal prediction provides a means for providing a prediction set that with a predetermined probability $1-\alpha$ covers the true label for a finite sample. Given a training data set $Z_{1:n}$, and a test data instance $X_T$, we obtain the conformal prediction set $\hat{C}(X_T)\subseteq \{1,\dots,K\}$, which satisfies
\begin{equation}\label{UncondCov}
 \mbbP (Y_T \in \hat{C} (X_T) ) \geq 1 - \alpha,
 \end{equation}
In what follows, we describe a variant, known as split conformal prediction \citep{papadopoulos2002inductive, vovk2005algorithmic}, where the entire training data is split into two parts, indexed by $\mathcal{S}_1$, $\mathcal{S}_2$. The first part is used to estimate the score function $S(\cdot,\cdot)$, whose arguments consist of a point $(x,y)$, and some dataset $D$. A high value of $S((x,y), D)$ indicates that the point $(x,y)$ ``conforms'' to $D$. Then we evaluate the score function on the second part to obtain the conformity scores $V_i^{(x,y)} = S(Z_i,Z_{\mathcal{S}_1})$, for all $i\in\mathcal{S}_2$. In binary classification where $y\in\left\{0,1\right\}$, \citet{lei2014classification} proposed split-conformal classification with a class-specific coverage guarantee:
\begin{equation}\label{ClasscondCov}
\mbbP (Y_T \in \hat{C}(X_T) | Y_T = j ) \geq 1 - \alpha, j = 0,1.
\end{equation}
Here, the score function is chosen to be an estimate of the posterior class probability $\hat\eta_j(x) = \hat P(Y=j|X=x)$ based on a classification algorithm trained on the first half of the data, known as the training set; $\hat\eta_j$ is then evaluated on the second half of the data, known as the calibration set, resulting in conformity scores $V_i^{(x,y)} = \mathds{1}[y_i=1] \hat\eta_1(x_i)+\mathds{1}[y_i=0] \hat\eta_0(x_i)$ for all $i$ in the second half. Finally, the set-valued prediction $\hat C(x)$ is defined as 
\[\hat C(x)=\{j\in\{0,1\}:\hat\eta_j(x) \geq \text{Quantile}(\alpha;V_{\mathcal{I}_j}\cup\{\infty\})\}\]
where $\mathcal{I}_j$ is the index set of those points in the second part that belong to class $j$. If the class-specific coverage \eqref{ClasscondCov} is valid for both classes, then the marginal coverage \eqref{UncondCov} is automatically obtained.

Both the original and the split conformal prediction assume that the distributions of the test data and the training data are the same. \citet{tibshirani2019conformal} generalized conformal prediction for regression to the $\CS$ setting. Assume that the probability measure of the target data covariates is absolutely continuous with respect to that of the source data covariates, we consider using the Radon-Nikodym derivative $w(x) = dQ_X(x)/dP_X(x)$ as a way to augment the target data using the source data; in particular, define $p_i = w(X_i)/[\sum_{i'=1}^n w(X_{i'})+w(X_T)]$, for $i=1,\dots,n$ and $p_{T} = w(X_T)/[\sum_{i'=1}^n w(X_{i'})+w(X_T)]$; here $n$ is the total sample size for full conformal prediction, or the calibration sample size for split-conformal prediction. We can now use weighted quantile of the scores computed in the source data as the cutoff value, with $p_i$ as the weight: 
\begin{align*}
    \hat{C}(x) = \Biggl\{y\in \R: &V_{n+1}^{(x,y)} \geq \\ &\text{Quantile}\Bigl(\alpha; \sum_{i=1}^{n}  p_i\delta_{V_i^{(x,y)}}+ p_{n+1}\delta_{\infty}\Bigr)\Biggr\},
\end{align*}
where $\delta_c$ is a Dirac measure placing a point mass at $c$.\footnote{Here with a slight abuse of notation, we use a probability measure, instead of a set of numbers, as the second argument of the function $\text{Quantile}(\cdot;\cdot)$. In the standard, unweighted, sample quantile case, one could use the empirical measure associated with the set of numbers as the second argument.} \citet{tibshirani2019conformal} showed that $\hat{C}(x)$ satisfies the same coverage guarantee as in \eqref{UncondCov} assuming that the true value of $w(x)$ is known.

Our work differs from that of \citet{tibshirani2019conformal} in three aspects. First, we consider a more general distributional difference setting, namely the $\CSPD$, as opposed to the $\CS$ setting in their work. Second, we consider the cases in which there are multiple labeled target data points available for training, whereas in \citet{tibshirani2019conformal} there is only one target data point, unlabeled, as the test point data. Third, we consider the multi-category classification problem instead of the regression problem.
\subsection{Covariate Shift with Posterior Drift}\label{subsec:CSPD}
Classic and well-studied distributional assumptions between the source data and the target data include covariate shift ($\CS$) and posterior drift ($\PD$). Restrict our discussion to classification problems. $\CS$ assumes that for each class $j$, we have $\eta_{P,j}=\eta_{Q,j}$, but $P_X$ is allowed to be different from $Q_X$. $\PD$ assumes
$P_X=Q_X,$ but allows
$\eta_{P,j}(x)$ to be different from $\eta_{Q,j}(x).$ 

\citet{scott2019generalized} combined the $\CS$ and $\PD$ assumptions into the \textit{covariate shift with posterior drift} ($\CSPD$) assumption.  $\CSPD$ only assumes $\eta_{P,1}(x) = \phi(\eta_{Q,1}(x))$ for some strictly increasing function $\phi$, for all $x$. Compared to $\CS$, $\CSPD$ relaxes the requirement that $\eta_{P,1}=\eta_{Q,1}$; compared to $\PD$, $\CSPD$ dropped the requirement that $P_X=Q_X$. Both $\CS$ and $\PD$ are special cases of $\CSPD$. \citet{cai2021transfer} considered a special case of $\PD$ using a specific $\phi_j$ functions. Note that the original definition of $\CSPD$ in \citet{scott2019generalized} was restricted to binary classification; in this article, we study $\CSPD$ for the more general multicategory classification problem: in particular, we say that class $j$ satisfies the $\CSPD$ assumption if $\eta_{P,j}(x) = \phi_j(\eta_{Q,j}(x))$ for some strictly increasing function $\phi_j$.

At first appearance, $\CSPD$ may seem to be strong. However, it only asserts that for two points $x_1$ and $x_2$, if the probability that $x_1$ belongs to class $j$ is greater than that of $x_2$, assuming both are from the source distribution, then the same conclusion can be said if both are instead from the target distribution. In the binary classification setting, recall that the Bayes classifier is characterized by comparing $\eta_1(x)$ with $1/2$ \citep{lei2014classification, cai2021transfer}. In this setting, while $\CSPD$ allows the classification boundaries to differ between the source and the target distributions (since $\eta_{P,j}$ can be different from $\eta_{Q,j}$), it does ensure that if two points in the source are classified to be from different classes by the Bayes rule, then the Bayes rule would have the prediction if they are instead from the target distribution. 

The monotonicity assumption can also intuitively be understood as a type of invariance in the ranking of posterior probabilities between the source and target domains. For example, in medical diagnosis, when comparing the posterior probabilities of two patients getting a certain disease, the patient (e.g. someone who smokes) with a higher probability (than the other patient - perhaps someone who does not smoke) in the source population would remain to be the one with a higher probability if they both were in the target population. Intuitively, this means that a distribution shift does not fundamentally change the implications of covariates on the outcome. 

\section{METHODOLOGY}\label{sec:method}
\subsection{Weighted Conformal Classification under Covariate Shift with Posterior Drift}\label{subsec:methodwcp}

When there is a relatively small quantity of labeled data from the target distribution $Q$, how to leverage abundant data from a different source distribution $P$, if $\eta_{P,j}(x)\neq \eta_{Q,j}(x)$, poses a major challenge. As a reminder, we have a source sample with size $m$, $Z_{1:m}$, a target sample with size $n$, $Z_{(m+1):(m+n)}$,  and a test data point $Z_T$ from the target distribution. Both $Z_{1:m}$ and $Z_{(m+1):(m+n)}$ may be used for the training purpose. Our goal here is to construct a prediction set $\hat{C}(x)$ that ensures coverage guarantees for the test target data: 
\begin{equation}\label{cov1}
\mathbb{Q}\left(Y_T\in \hat{C}(X_T) |Y_T = j\right)\geq 1 - \alpha
\end{equation}
where $\mathbb{Q}$ is the product measure of the measure $P^m$ governing the distribution of the source sample $Z_{1:m}$, and the measure $Q^{n+1}$ governing the distributions of both the target sample $Z_{(m+1):(m+n)}$ and the new test target data point $Z_T$.

Following the work of \citet{lei2014classification}, a natural choice of the prediction set is \[\hat{C}(x)=\left\{j:\hat{\eta}_{Q,j}(x)>\hat{t}_{j,\alpha}^* \right\},\] that is, based on comparing the estimated conditional class probabilities $\hat{\eta}_{Q,j}(x)$ with a threshold $\hat{t}_{j,\alpha}^*$. One challenge is that in the absence of sufficiently many labeled data from the target distribution $Q$, the estimation of $\eta_{Q,j}(x)$ may be difficult. 

If the $\CSPD$ assumption holds for class $j$, we have $\eta_{P,j}(x)=\phi_j(\eta_{Q,j}(x))$, where $\phi_j$ is an monotone increasing function. Hence, thresholding $\eta_{Q,j}(x)$ is equivalent to thresholding $\eta_{P,j}(x)$ as long as the thresholds are chosen properly, that is,
\begin{equation}\label{cov2}
\left\{ x: \eta_{Q,j}(x) \geq t^*_{j,\alpha} \right\} = \left\{ x: \eta_{P,j}(x) \geq t_{j,\alpha} \right\},
\end{equation}
where $t^*_{j,\alpha}$ and $t_{j,\alpha}$ satisfy that $\phi_j(t^*_{j,\alpha})=t_{j,\alpha}.$
Therefore, one may instead construct the prediction set as 
\begin{equation}\label{cov3}
\hat{C}(x)=\left\{j:\hat{\eta}_{P,j}(x)>\hat{t}_{j,\alpha} \right\}.
\end{equation}
Note that there are many labeled source data that afford an accurate estimation $\hat{\eta}_{P,j}$. 

After $\hat{\eta}_{P,j}(x)$ is chosen as the score, the next question is how to select the threshold $\hat{t}_{j,\alpha}$. Given the prediction set \eqref{cov3}, the coverage guarantee for the test data point \eqref{cov1} becomes \[\mathbb{Q}\left(\hat{\eta}_{P,j}(X_T)\geq \hat{t}_{j,\alpha}|Y_T = j \right) \geq 1 - \alpha, \forall j.\]
Following the conformal prediction literature, a standard choice of $\hat{t}_{j,\alpha}$ is the $\alpha$ quantile of $\hat{\eta}_{P,j}(X)$ over all the target data points from class $j$. Note that this would mean that the class label should be available, hence we will need labeled target data. The practical difficulty is that labeled target data are either not available at all, as in the setting of \citet{tibshirani2019conformal}, or insufficient, as in our setting.

Fortunately, there are abundant labeled data from the source distribution. If we know the ratio of covariate likelihoods $dQ_{X|Y=j}/dP_{X|Y=j}$ for each class $j$, we can modify the conformal prediction procedure \citep{tibshirani2019conformal}. The idea is to calculate the quantile based on a weighted empirical distribution that includes both the target data and the source data, leading to Weighted Conformal Classification under $\CSPD$. 

Algorithm \hyperref[alg1]{1} below sketches our main methodology. Denote $\mathcal{I}^j$ as the index set for class $j$. We first divide the source training data into two splits, indexed by $\mathcal{S}_1$ and $\mathcal{S}_2$. The first split is used to estimate the conditional class probability $\hat{\eta}_{P,j}$ for all $j$. Both the second split of the source sample and the target sample, indexed by $\mathcal{T}$, form the calibration set. For each class $j$, let $\mathcal{R}^j = (\mathcal{S}_2\cap \mathcal{I}^j)\cup (\mathcal{T}\cap \mathcal{I}^j)$ denote the index for the calibration set. We then evaluate $\hat{\eta}_{P,j}$ on all class $j$ points from $\mathcal{R}^j\cup \{T\}$, then compute the threshold $\hat{t}_{j,\alpha}$ as the $\alpha$ quantile of the weighted conformity scores. 

\begin{algorithm}\label{alg1}
\caption{Weighted Conformal Classification under $\CSPD$ (WCC-CSPD)}
\label{alg:weighted_conformal_classification}
\begin{algorithmic}
\STATE \textbf{Input:} $Z_{1:m}$ from $P$, $Z_{(m+1):(m+n)}$ from $Q$, coverage rate $1-\alpha$, a classifier $\mathcal{A}$ for estimating $\hat{\eta}_{P,j}$, and a test data point $X_T = X_{m+n+1}$ from $Q$.
\STATE \textbf{Output:} A set-valued classifier $\hat{C}(x)$ that predicts the class label.
\end{algorithmic}
\begin{algorithmic}[1]
\STATE Randomly split $\left\{1,\ldots, m\right\}$ into two equal sized subsets, indexed by $\mathcal{S}_1,\mathcal{S}_2$
\STATE Estimate $\hat{\eta}_{P,j}$ using data set $\left\{(X_i,Y_i):i\in\mathcal{S}_1\right\}$ and classifier $\mathcal{A}$
\FOR{$j\in \left\{1,\dots,K\right\}$}
    \STATE Evaluate $\hat{\eta}_{P,j}(x)$ on $\mathcal{R}^j \cup \{T\}$ where $T=m+n+1.$
    \STATE Compute the weights $\tilde w_{ij}$ and $\tilde w_{Tj}$ according to \eqref{weight}.
    \STATE Compute the threshold $\hat{t}_{j,\alpha}$ according to (\ref{threshold}).
\ENDFOR
\RETURN $\hat{C}(x)=\left\{j:\hat{\eta}_{P,j}(x)>\hat{t}_{j,\alpha} \right\}$
\end{algorithmic}
\end{algorithm}

Assuming $Q_{X|Y=j}$ is absolutely continuous w.r.t $P_{X|Y=j}$, denote the Radon-Nikodym derivative $w^j(x) = dQ_{X|Y=j}(x)/dP_{X|Y=j}(x)$. To compute the weight for each data point in $\mathcal{S}_2\cup \mathcal{T}$, we first define a series of \textit{initial weight functions}, one for each data point: $w_{ij}(x) = 1$ for data point $i\in\mathcal{S}_2\cap\mathcal{I}^j$; and $w_{ij}(x) = w^j(x)$ for data point $i\in\mathcal{T}\cap\mathcal{I}^j$. We then assign a weight of $\tilde{w}_{ij}$ to data point $i$ in class $j$, defined as,
\begin{equation}\label{weight}
\tilde{w}_{ij} = \frac{\sum_{\sigma :\sigma(T) = i}\prod_{k}w_{kj}(x_{\sigma(k)})}{\sum_{\sigma}\prod_{k}w_{kj}(x_{\sigma(k)})},
\end{equation}
where $T = m+n+1$ is the index for the test data point, and $\sigma$ is a permutation of the indices  $\mathcal{R}^j\cup \{T\}$ for those class $j$ points among both the calibration set and the test data point. Here, the products are taken over all points $k\in \mathcal{R}^j \cup \{T\}$ and the summations are taken over all permutations. The weights are defined as in \eqref{weight} to ensure coverage guarantees, as shown in Lemma \ref{lemma1} below.

\begin{lemma}\label{lemma1}
Obtain the set-valued classifier based on Algorithm \hyperref[alg1]{1}, that is, estimate $\hat{\eta}_{P,j}$ using the first split of the source training data $\mathcal{S}_1$ and classifier $\mathcal{A}$, evaluate $\hat{\eta}_{P,j}(x)$ on $\mathcal{R}^j\cup \{T\}$ (the second split of the class $j$ source training data, the target training data, and the test data point), compute the weights $\tilde w_{ij}$ and $\tilde w_{Tj}$ according to \eqref{weight}, and \begin{equation}\label{threshold}
    \hat{t}_{j,\alpha}=\text{Quantile} \left(\alpha, \sum_{{i\in \mathcal{R}^j}}\tilde{w}_{ij} \delta_{\hat{\eta}_{P,j}(x_i)}+\tilde{w}_{Tj}\delta_{\infty}\right).
\end{equation}
Then we have,
\[
\mathbb{Q}\left( \hat{\eta}_{P,j}(X_{T}) \geq \hat{t}_{j,\alpha} | Y_{T} = j \right) \geq 1-\alpha,
\]
where $\mathbb{Q}$ is with respect to all the data points in $\mathcal{S}_2\cup \mathcal{T}\cup \{T\}$.
\end{lemma}
All proofs are provided in Appendix A. Lemma \ref{lemma1} establishes the groundwork for incorporating both source and target samples in the calculation of weighted quantiles. Theorem \ref{thm1} below shows a lower bound of the coverage probability of the prediction set induced by the threshold in Lemma \ref{lemma1}.

\begin{theorem}\label{thm1}
Assume for every class $j$, $Q_{X|Y=j}$ is absolutely continuous w.r.t. $P_{X|Y=j}$, and the $\CSPD$ assumption holds. For a test point $(X_T, Y_T)$, let $\hat{C}(X_T)$ be the set-valued classifier obtained from Algorithm \hyperref[alg1]{1}. Then $\forall\alpha\in (0,1),$ we have
\[\mathbb{Q}\left(Y_T \in \hat{C}(X_T) \right) \geq 1-\alpha,\]
where $\mathbb{Q}$ is with respect to all the data points in $\mathcal{S}_2\cup \mathcal{T}\cup \{T\}$.
\end{theorem}

\subsection{The Computational Issue of the Weights}
Under the pure $\CS$ setting, as explored by \citet{tibshirani2019conformal}, the scenario is simplified to no labeled target data for training ($n=0$) and $w_{ij}(x)=1$ for all $i$ being a labeled source data, and $w_{Tj}(x)=w^j(x)$. Hence, the weight in \eqref{weight} is simplified to 
\[\tilde{w}_{ij} = \frac{\sum_{\sigma :\sigma(T) = i}w^j(x_i)}{\sum_{\sigma}w^j(x_{\sigma(T)})} = \frac{w^j(x_i)}{\sum_{i\in\mathcal{R}^j\cup \{T\}}w^j(x_i)}\]
While this simplification under the $\CS$ setting streamlines the weight calculation, it overlooks the potentially rich information in the target data. Our work has incorporated the target training data to capitalize on the full potential of label information therein.

Since the denominator in \eqref{weight} is the same for all data points, we now focus on the numerator of (\ref{weight}), $\sum_{\sigma :\sigma(T) = i}\prod_{k}w_{kj}(x_{\sigma(k)})$. Here one takes the sum over many permutations of $N_j = |\mathcal{R}^j\cup \{T\}|$ many elements. The first $N_j^S = |(\mathcal{S}_2\cap \mathcal{I}^j)|$ entries of the permutation are evaluated using the initial weight function $w_{ij}(x) = 1$, and the remaining $N_j^T+1 = |(\mathcal{T}\cap \mathcal{I}^j)\cup \{T\}|$ entries are evaluated with the $w_{ij}(x) = w^j(x)$ function. The product of these terms forms one term in the numerator of \eqref{weight}. Then $\tilde{w}_{ij}$ is the summation of these product terms over all possible permutations $\sigma$ such that $\sigma(T) = i$, scaled by the common denominator. For each $i$, while there are $(N_j -1)!$ many such permutations which satisfy $\sigma(T) = i$, many of them lead to the same product, and hence, the unique number of product terms that one needs to compute is reduced by a factor of $N_j^S! \cdot N_j^T!$ times. In other words, for each $i$ in class $j$, we ``only'' need to consider the ``$(N_j -1)$ choose $N_j^T$'' many combinations (instead of permutations) out of $\left\{w^j(x_i): i\in \mathcal{R}^j\cup \{T\} \right\}$, compute the product in each combination, and sum the products over all combinations, followed by a normalization term.

The number of combinations required for each $i$ remains excessively large. In this article, we leverage Newton's identities \citep{littlewood1970university} to make the computation manageable.

\begin{lemma}\label{lemma2}
Let $a_1,\dots,a_n$ be $n$ numbers. For $n > k\geq 1$, denote by $p_k(a_1,\dots,a_n)$ the $k$-th power sum:
$p_k(a_1,\dots,a_n) = \sum_{i=1}^n a_i^k.$
For $k\geq 0 $, denote by $e_k(a_1,\dots,a_n)$ the elementary symmetric polynomial (the sum of all distinct products of $k$ distinct variables),
\begin{align*}
e_0 (a_1,\dots,a_n) &= 1\\
e_1 (a_1,\dots,a_n) &= a_1 + a_2 + \cdots + a_n\\
e_2 (a_1,\dots,a_n) &= \sum_{1\leq i < j \leq n} a_i a_j,\\
&\ \ \vdots \\
e_n (a_1,\dots,a_n) &= a_1 a_2 \cdots a_n\\
\end{align*}

Then we have Newton's identities: for all $1\leq k \leq n$,
\begin{equation}\label{Newton}
\begin{aligned}
   e_k&(a_1,\dots, a_n) \\
   &= \frac{1}{k} \sum_{i=1}^k (-1)^{i-1} e_{k-i}(a_1,\dots,a_n) p_i(a_1,\dots,a_n).
\end{aligned}
\end{equation}
\end{lemma}

Following Lemma \ref{lemma2}, one can show that the numerator of \eqref{weight}, $\sum_{\sigma :\sigma(T) = i}\prod_{k}w_{kj}(x_{\sigma(k)})$, is the same as  $w^j(x_{i})\cdot e_{N_j^T}(\left\{w^j(x_c): c \neq i, c\in \mathcal{R}^j\cup \{T\} \right\})$, that is, $w^j(x_c)$ times the $N_j^T$-th elementary symmetric polynomial over the set of ${w^j(x_c)}$ for $c \neq i$ and $c\in \mathcal{R}^j\cup {T}$. Through this pivotal simplification, we convert the computationally intensive task of calculating \eqref{weight} into a more tractable problem, by using the recursive formula \eqref{Newton}. It is particularly beneficial when the sample size is large. 

\subsection{Generalized Covariate Shift with Posterior Drift}\label{subsec:gcspd}
In this section, we expand upon the $\CSPD$ assumption to a more general and flexible assumption named Generalized $\CSPD$ ($\gCSPD$).

\begin{definition}[\textbf{$\gCSPD$ at $\alpha$}]
 We say class $j$ satisfies $\gCSPD$ at $\alpha$ if, for some $\alpha\in (0,1)$ and some function $\phi_j$, the following conditions hold: For any $t_1, t_2$ such that $t_1 < t_{j,\alpha} < t_2$, where $t_{j,\alpha}$ satisfies $\mathbb{Q}(\{ x: \eta_{P,j}(x) \geq t_{j,\alpha} \}) = \alpha$, we have $\phi_j(t_1) < \phi_j(t_{j,\alpha})< \phi_j(t_2)$ and $\eta_{P,j}(x) = \phi_j(\eta_{Q,j}(x))$.
\end{definition}
        
\begin{figure}[H]
    \centering
    \includegraphics[width=2.3in]{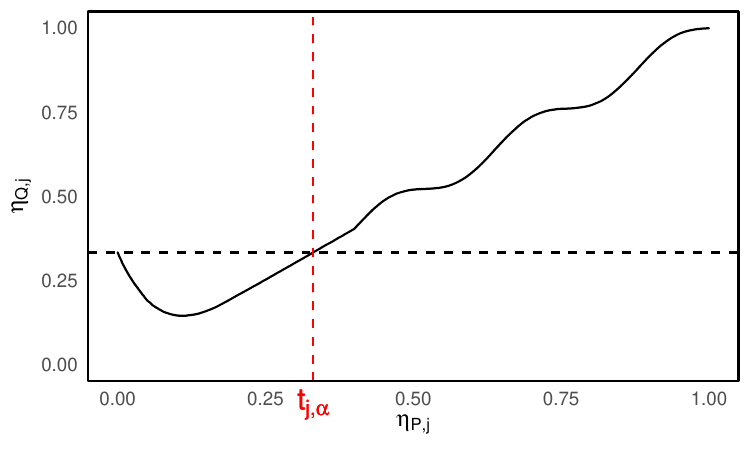} 
    \caption{Illustration of $\gCSPD$ at $\alpha$}
    \label{g-illu}
\end{figure}

The $\gCSPD$ assumption relaxes the monotone assumption of $\phi_j$ in the original $\CSPD$ framework, allowing for additional modes of shifts between the source and the target distributions. Figure \ref{g-illu} is an illustration of a $\gCSPD$ at level $\alpha$. The solid black curve shows the shift between $\eta_{P,j}$ and $\eta_{Q,j}$ as defined by function $\phi_j$. The red dashed vertical line indicates the threshold $t_{j,\alpha}$, where the probability that $\eta_{P,j}(X)$ exceeds this threshold equals $\alpha$.

The intuition behind the relaxation of the monotonicity of $\phi_j$ in $\gCSPD$ is the following: the key of the $\CSPD$ assumption is that it allowed us to replace the thresholding inequality $\eta_{Q,j}(x)\geq t^*_{j,\alpha}$ by a new inequality $\eta_{P,j}(x)\geq t_{j,\alpha}$ with $\phi_j(t^*_{j,\alpha})=t_{j,\alpha}$; see (\ref{cov2}). To this end, we first need to have $\eta_{P,j}(x) = \phi_j(\eta_{Q,j}(x))$ at the threshold values $\eta_{Q,j}(x) = t^*_{j,\alpha}$ and $\eta_{P,j}(x) = t_{j,\alpha}$; in addition, in order for the coverage probability to remain unchanged as we switch the thresholding inequality, we need $\phi_j(t_1) < \phi_j(t_{j,\alpha})<\phi_j(t_2)$, so that no data instances that satisfied $\eta_{Q,j}(x) \geq t_{j,\alpha}^*$ would turn out to be $\eta_{P,j}(x) < t_{j,\alpha}$ after the switch, and vice versa. Theorem \ref{thm1} still remains true with the replacement of the $\CSPD$ assumption by $\gCSPD$ at $\alpha$. Note that if $\gCSPD$ is satisfied at all $\alpha\in (0,1)$, then $\CSPD$ is satisfied.



\begin{figure*}[!t]
    \centering
    \includegraphics[width = \textwidth]{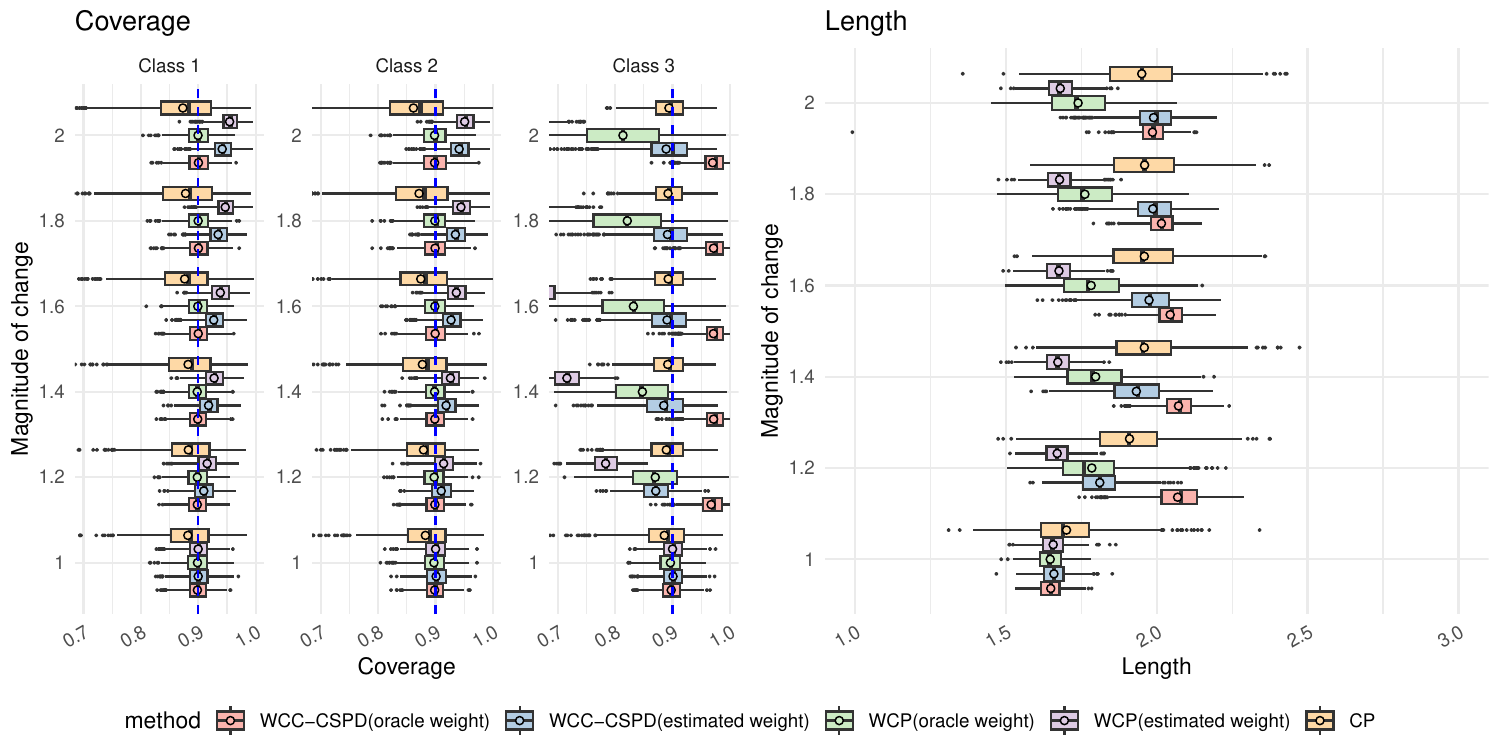}
    \caption{Performance of all baselines in the simulation setup described in Section \ref{sec:simulation}. The blue vertical lines correspond to the target coverage level ($1 - \alpha = 0.9$). The average marginal coverage rates (over all $r$) are: WCC-CSPD (oracle weight, estimated weight) at 0.932, 0.903; WCP (oracle weight, estimated weight) at 0.869, 0.811; CP at 0.885.} 
    \label{fig5}
\end{figure*}

\section{ASYMPTOTIC PROPERTIES}\label{sec:asymp}
Denote the (true) confidence set of class $j$ at level $\alpha$  to be $C_{j,\alpha}:= \left\{x:\eta_{P,j}(x) \geq t_{j,\alpha}\right\}$. This may be viewed as the dual form of the set-valued classifier. Specifically, the set-valued classifier $\hat{C}(x)=\left\{j:\hat{\eta}_{P,j}(x)>\hat{t}_{j,\alpha} \right\}$ in Algorithm \hyperref[alg1]{1} is induced by $\hat{C}_{j,\alpha} = \left\{x:\hat{\eta}_{P,j}(x) \geq \hat{t}_{j,\alpha}\right\}$. To evaluate the performance of $\hat{C}_{j,\alpha}$ as an estimate of $C_{j,\alpha}$, we focus on measuring $\hat{C}_{j,\alpha}\triangle C_{j,\alpha}$ in this section, where $\hat{C}\triangle C := (\hat{C}\setminus C)\cup (C\setminus \hat{C})$ denotes the symmetric difference between $\hat{C}$ and $C$.

Denote $G_j(t)=Q_j(\left\{ x: \eta_{P,j}(x)\leq t\right\})$ the empirical distribution function of the random variable $\eta_{P,j}(X)$. Let $\mathbb{Q}_j$ denote the probability measure under $G_j.$ Consider the following assumptions:\\
\textbf{(A)}. $\hat{\eta}_{P,j}$ is a $(\delta_m, \theta_m)$-accurate estimator: $P^m(\|\hat{\eta}_{P,j} - \eta_{P,j}\|_\infty \geq \delta_m)\leq \theta_m \forall j$ as $m\rightarrow\infty.$ \\
\textbf{(B)}. There exist constants $b_1, b_2, d_0$ and $\lambda >0$ such that for all $d\in\left[-d_0,d_0\right]$, 
\[b_1 |d|^\lambda \leq | G_j(t_j+d)-G_j(t_j) | \leq b_2 |d|^\lambda, \forall j\]
Assumption \textbf{(A)} \citep{lei2014classification} requires $(\delta_m, \theta_m)$-accurate estimators. Specific examples of such estimators with explicit rates for $\delta_m$ and $\theta_m$, over a broad class of models, can be found in \citep{audibert2007fast, van2008high}, including the local polynomial regression and $l_1$-penalized logistic regression. Assumption \textbf{(B)}, which is a version of the margin assumption (MA) \citep{audibert2005fast}, suggests that there are few data points near the threshold.

\begin{theorem}\label{thm2}
Define $m = |\mathcal{S}_1|$ as the size of data set used to estimate $\hat{\eta}_{P,j}$, and $n_{j} = |\mathcal{R}^j|$ as the size of calibration set for class $j$. Under $\gCSPD$ assumption at $\alpha$ and assumptions \textbf{(A)}, \textbf{(B)}, if $Q_{X|Y=j}$ is absolutely continuous with respect to $P_{X|Y=j}$ and $\tilde{w}_{ij}$ are bounded, then for each $r>0$, there exists a positive constant $c$ such that for $m$ and $n_{j}$ large enough, with probability at least $1-\theta_{m}-n_{j}^{-r}$
\[\mathbb{Q}_j(\hat{C}_{j,\alpha}\triangle C_{j,\alpha})\leq c\left\{ \delta_{m}^\lambda  + \left( \frac{\log n_{j}}{n_{j}}\right)^{\frac{1}{2}}\right\}.\]
\end{theorem}
Theorem \ref{thm2} establishes a convergence rate comparable to those found in related estimation and classification problems \citep{lei2014classification, sadinle2019least, scott2019generalized}, with the key difference being the use of $n_j$ instead of $n$. The advantage of our approach lies in the fact that $n_j$ is larger, as it incorporates both the source and target samples during calibration, which leads to faster convergence. In contrast, previous work either excludes the target sample in the calibration step \citep{tibshirani2019conformal} or excludes the source sample during calibration \citep{scott2019generalized}, resulting in a smaller sample size and slower convergence.

\begin{figure*}[!t]
    \centering
    \includegraphics[width=\linewidth]{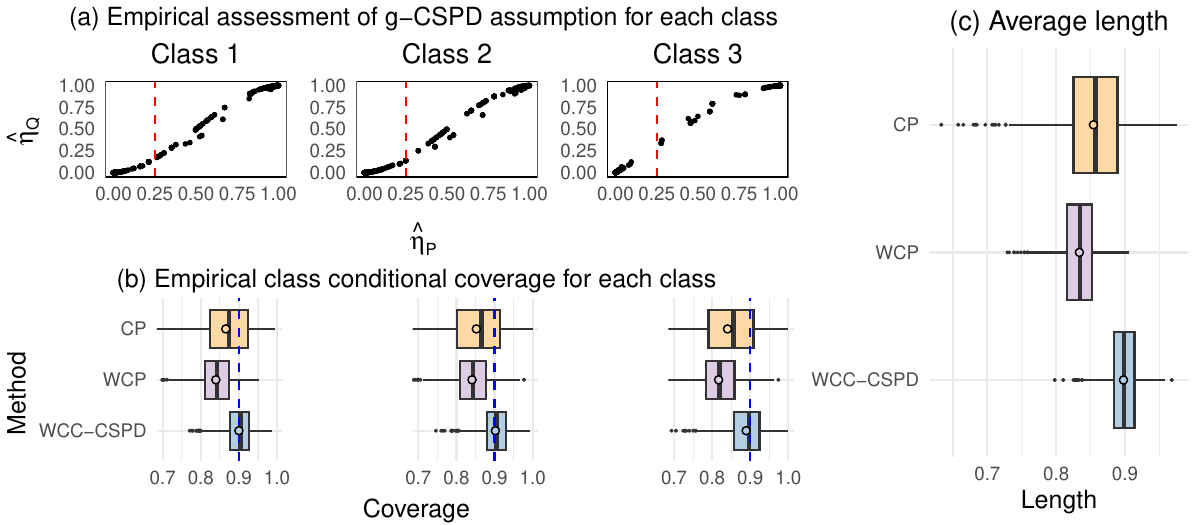}
    \caption{Performance of all baselines in the semi-synthetic setup described in Section \ref{sec:simulation}. In (a), the red dashed line illustrates that $\gCSPD$ holds at the corresponding level. The marginal coverage rates are: WCC-CSPD at 0.898, WCP at 0.834, CP at 0.854.}
    \label{fig6}
\end{figure*}
\section{NUMERICAL STUDIES}\label{sec:simulation}
\textbf{Simulation study.} We create synthetic datasets of covariates $X\in\mathbb{R}^{5}$ and class $Y\in\left\{1,2,3\right\}$. For each trial, we first sample 3000 points from the following distributions (1000 from each class),
\[x_1\sim\begin{cases}
N(-3,1) & \text{ if } y= 1\\
N(-2,1) & \text{ if } y= 2\\
N(0,1) & \text{ if } y= 3
\end{cases}\ ,x_j\sim N(0,1), j=2,\dots,5\]
Then we randomly partition the data into two equal-sized subsets. The first half serves as the source training sample. In the second half, we compute 
$\eta_{Q,j}(x) = \phi_j^{-1}(\eta_{P,j}(x))$ given a function $\phi_j$, and then relabel the class accordingly. For $j = 1,2$, we choose $\phi^{-1}_j(t) = t^r$ and we simulate different scenarios by varying the exponents $r$ to represent different magnitudes of shift (see the y-axis of the left panel in Figure \ref{fig5}). Under this setting, it can be verified that the posterior for the third class $\eta_{Q,3}(x) = 1 - \eta_{Q,1}(x) - \eta_{Q,2}(x)$ also satisfies the $\CSPD$ assumption. To mimic scenarios where only a few labeled target data points are available for training, we further split the target sample into training set and test set. The ratio of source training, target training, and target test sets is set to be $5:1:4$.  Throughout the experiments, we repeat each scenario for 1000 trials and set $\alpha = 0.1$.

\textbf{Semi-synthetic dataset.} We also consider the Maternal Health Risk data from UCI Machine Learning Repository \citep{misc_maternal_health_risk_863}, which has $N = 1013$ instances, consisting of a class label $Y$ and feature $X$ with $6$ dimensions. We use a similar strategy as stated in the simulation study to partition the data and reassign the labels, ensuring the $\gCSPD$ assumption holds for all classes and for most of the $\alpha$ levels (See Figure \ref{fig6} (a)). 
Details of the data description and the label generation can be found in Appendix B.

\textbf{Baselines.} We examine weighted conformal prediction (WCP) as proposed in \citet{tibshirani2019conformal}. The original WCP algorithm, designed for regression problems, does not incorporate labeled target data in the training procedure. We also explore the split conformal prediction (CP) \citep{vovk2005algorithmic, lei2014classification}, which is trained exclusively on the labeled target data.

For our proposed method (WCC-CSPD) and WCP, we evaluate both scenarios where we have the oracle weight (only for simulation datasets) and where we estimate the initial weight function $w^j(x)$ using the training source and target samples (the same strategy as \citep{tibshirani2019conformal}.) Details of the weight estimation are provided in Appendix B. For all methods, we use the same model (XGBoost) for estimating $\hat{\eta}_{P,j}$ and $\hat{w}^j(x)$. We assess performance using two commonly adopted metrics: class-conditional coverage and average prediction set length. Class-conditional coverage measures the proportion of instances for which the true class label is contained within the prediction set. Average prediction set length refers to the expected size of the prediction set.

\textbf{Results.} In all scenarios where the $\CSPD$ assumption holds, WCC-CSPD consistently achieved the desired $1-\alpha$ coverage, regardless of whether the oracle weight was available or not. Additionally, even when the marginal coverage of two baselines falls short, the average prediction length of WCC-CSPD are not significantly larger. In Figure \ref{fig5}, when there is no shift (i.e., $r=1$), WCC-CSPD performs similarly to WCP. However, when using only target data, CP, due to its limited data points (only 100 per class $j$), tends to under-cover and meanwhile produces the largest average length. As the shift magnitude $r$ increased, WCC-CSPD ensures both class-conditional and marginal coverage, whereas WCP fails to cover. Focusing on the third class, since the probability of a target sample being relabeled as the third class $\eta_{Q,3}=1-\eta_{Q,1}-\eta_{Q,2}$ increased with $r$, more target points were relabeled into class 3. This results in an improvement in CP's performance for the third class as $r$ increased. However, due to the increasing change in the third class, the effective sample size for WCP decreases as it is trained only on source data, resulting in worsening performance. Unlike WCP, WCC-CSPD does not suffer from a reduced effective sample size due to the inclusion of more target data points, thereby maintaining sufficient coverage. Figure \ref{fig6}(a) shows that when the $\gCSPD$ assumption is satisfied, the marginal coverage for WCC-CSPD is 0.898. This slight deviation from the intended $1-\alpha$ coverage is reasonable, considering the use of estimated weights. In contrast, due to a smaller sample size compared to the simulation, CP fails to reach desire coverage level. WCP suffers from a smaller effective sample size.

\section{CONCLUSION}\label{sec:conclude}
We present a comprehensive framework for weighted conformal classification under the $\CSPD$ assumption. Our proposed method alleviates the computational issue and can leverage abundant labeled source data alongside scarce target data to construct prediction sets with desired coverage guarantees. The theoretical contributions of this work, supported by rigorous proofs and empirical validations through simulations and semi-synthetic experiments, underline the effectiveness of our methods in achieving the desired coverage probabilities under $\CSPD$ and $\gCSPD$ assumption. Future work will explore the multi-source scenarios (see Appendix C for discussion), further generalizations of the $\CSPD$, and the validation of the $\CSPD$ assumption.

\bibliographystyle{plainnat}
\bibliography{Conformal_Prediction_Under_Generalized_Covariate_Shift_with_Posterior_Drift}

\begin{thebibliography}{47}
\providecommand{\natexlab}[1]{#1}
\providecommand{\url}[1]{\texttt{#1}}
\expandafter\ifx\csname urlstyle\endcsname\relax
  \providecommand{\doi}[1]{doi: #1}\else
  \providecommand{\doi}{doi: \begingroup \urlstyle{rm}\Url}\fi

\bibitem[Ahmed(2023)]{misc_maternal_health_risk_863}
Marzia Ahmed.
\newblock {Maternal Health Risk}.
\newblock UCI Machine Learning Repository, 2023.
\newblock {DOI}: https://doi.org/10.24432/C5DP5D.

\bibitem[Ai and Ren(2024)]{ai2024not}
Jiahao Ai and Zhimei Ren.
\newblock Not all distributional shifts are equal: Fine-grained robust
  conformal inference.
\newblock \emph{arXiv preprint arXiv:2402.13042}, 2024.

\bibitem[Audibert and Tsybakov(2005)]{audibert2005fast}
Jean-Yves Audibert and Alexandre~B Tsybakov.
\newblock Fast learning rates for plug-in classifiers under the margin
  condition.
\newblock \emph{arXiv preprint math/0507180}, 2005.

\bibitem[Audibert and Tsybakov(2007)]{audibert2007fast}
Jean-Yves Audibert and Alexandre~B. Tsybakov.
\newblock {Fast learning rates for plug-in classifiers}.
\newblock \emph{The Annals of Statistics}, 35\penalty0 (2):\penalty0 608 --
  633, 2007.

\bibitem[Barber et~al.(2023)Barber, Candes, Ramdas, and
  Tibshirani]{barber2023conformal}
Rina~Foygel Barber, Emmanuel~J Candes, Aaditya Ramdas, and Ryan~J Tibshirani.
\newblock Conformal prediction beyond exchangeability.
\newblock \emph{The Annals of Statistics}, 51\penalty0 (2):\penalty0 816--845,
  2023.

\bibitem[Ben-Tal et~al.(2013)Ben-Tal, Den~Hertog, De~Waegenaere, Melenberg, and
  Rennen]{ben2013robust}
Aharon Ben-Tal, Dick Den~Hertog, Anja De~Waegenaere, Bertrand Melenberg, and
  Gijs Rennen.
\newblock Robust solutions of optimization problems affected by uncertain
  probabilities.
\newblock \emph{Management Science}, 59\penalty0 (2):\penalty0 341--357, 2013.

\bibitem[Burnaev and Vovk(2014)]{burnaev2014efficiency}
Evgeny Burnaev and Vladimir Vovk.
\newblock Efficiency of conformalized ridge regression.
\newblock In \emph{Conference on Learning Theory}, pages 605--622. PMLR, 2014.

\bibitem[Cai and Wei(2021)]{cai2021transfer}
T~Tony Cai and Hongji Wei.
\newblock Transfer learning for nonparametric classification: Minimax rate and
  adaptive classifier.
\newblock \emph{The Annals of Statistics}, 49\penalty0 (1):\penalty0 100--128,
  2021.

\bibitem[Cauchois et~al.(2024)Cauchois, Gupta, Ali, and
  Duchi]{cauchois2024robust}
Maxime Cauchois, Suyash Gupta, Alnur Ali, and John~C Duchi.
\newblock Robust validation: Confident predictions even when distributions
  shift.
\newblock \emph{Journal of the American Statistical Association}, pages 1--66,
  2024.

\bibitem[Chen(2019)]{chen2019nearest}
George Chen.
\newblock Nearest neighbor and kernel survival analysis: Nonasymptotic error
  bounds and strong consistency rates.
\newblock In \emph{International Conference on Machine Learning}, pages
  1001--1010. PMLR, 2019.

\bibitem[Choi et~al.(2017)Choi, Fazekas, Sandler, and Cho]{choi2017transfer}
Keunwoo Choi, Gy{\"o}rgy Fazekas, Mark Sandler, and Kyunghyun Cho.
\newblock Transfer learning for music classification and regression tasks.
\newblock \emph{arXiv preprint arXiv:1703.09179}, 2017.

\bibitem[Duan et~al.(2012)Duan, Xu, and Tsang]{duan2012learning}
Lixin Duan, Dong Xu, and Ivor Tsang.
\newblock Learning with augmented features for heterogeneous domain adaptation.
\newblock \emph{arXiv preprint arXiv:1206.4660}, 2012.

\bibitem[Fannjiang et~al.(2022)Fannjiang, Bates, Angelopoulos, Listgarten, and
  Jordan]{fannjiang2022conformal}
Clara Fannjiang, Stephen Bates, Anastasios~N Angelopoulos, Jennifer Listgarten,
  and Michael~I Jordan.
\newblock Conformal prediction under feedback covariate shift for biomolecular
  design.
\newblock \emph{Proceedings of the National Academy of Sciences}, 119\penalty0
  (43):\penalty0 e2204569119, 2022.

\bibitem[Gama et~al.(2014)Gama, {\v{Z}}liobait{\.e}, Bifet, Pechenizkiy, and
  Bouchachia]{gama2014survey}
Jo{\~a}o Gama, Indr{\.e} {\v{Z}}liobait{\.e}, Albert Bifet, Mykola Pechenizkiy,
  and Abdelhamid Bouchachia.
\newblock A survey on concept drift adaptation.
\newblock \emph{ACM computing surveys (CSUR)}, 46\penalty0 (4):\penalty0 1--37,
  2014.

\bibitem[Gao et~al.(2007)Gao, Fan, Han, and Yu]{gao2007general}
Jing Gao, Wei Fan, Jiawei Han, and Philip~S Yu.
\newblock A general framework for mining concept-drifting data streams with
  skewed distributions.
\newblock In \emph{Proceedings of the 2007 siam international conference on
  data mining}, pages 3--14. SIAM, 2007.

\bibitem[Kpotufe and Martinet(2021)]{kpotufe2021marginal}
Samory Kpotufe and Guillaume Martinet.
\newblock Marginal singularity and the benefits of labels in covariate-shift.
\newblock \emph{The Annals of Statistics}, 49\penalty0 (6):\penalty0
  3299--3323, 2021.

\bibitem[Kulis et~al.(2011)Kulis, Saenko, and Darrell]{kulis2011you}
Brian Kulis, Kate Saenko, and Trevor Darrell.
\newblock What you saw is not what you get: Domain adaptation using asymmetric
  kernel transforms.
\newblock In \emph{CVPR 2011}, pages 1785--1792. IEEE, 2011.

\bibitem[Lei(2014)]{lei2014classification}
Jing Lei.
\newblock Classification with confidence.
\newblock \emph{Biometrika}, 101\penalty0 (4):\penalty0 755--769, 2014.

\bibitem[Lei and Wasserman(2014)]{lei2014distribution}
Jing Lei and Larry Wasserman.
\newblock Distribution-free prediction bands for non-parametric regression.
\newblock \emph{Journal of the Royal Statistical Society: Series B (Statistical
  Methodology)}, 76\penalty0 (1):\penalty0 71--96, 2014.

\bibitem[Lei et~al.(2018)Lei, G’Sell, Rinaldo, Tibshirani, and
  Wasserman]{lei2018distribution}
Jing Lei, Max G’Sell, Alessandro Rinaldo, Ryan~J Tibshirani, and Larry
  Wasserman.
\newblock Distribution-free predictive inference for regression.
\newblock \emph{Journal of the American Statistical Association}, 113\penalty0
  (523):\penalty0 1094--1111, 2018.

\bibitem[Lei and Cand{\`e}s(2021)]{lei2021conformal}
Lihua Lei and Emmanuel~J Cand{\`e}s.
\newblock Conformal inference of counterfactuals and individual treatment
  effects.
\newblock \emph{Journal of the Royal Statistical Society: Series B (Statistical
  Methodology)}, 2021.

\bibitem[Littlewood(1970)]{littlewood1970university}
D.E. Littlewood.
\newblock \emph{A University Algebra: An Introduction to Classic and Modern
  Algebra}.
\newblock Dover books on intermediate and advanced mathematics. Dover, 1970.
\newblock ISBN 9780486627151.
\newblock URL \url{https://books.google.com/books?id=4bgrAAAAYAAJ}.

\bibitem[Liu et~al.(2024)Liu, Levis, Normand, and Han]{liu2024multi}
Yi~Liu, Alexander~W Levis, Sharon-Lise Normand, and Larry Han.
\newblock Multi-source conformal inference under distribution shift.
\newblock \emph{Proceedings of machine learning research}, 235:\penalty0 31344,
  2024.

\bibitem[Ma et~al.(2015)Ma, Gong, and Mao]{ma2015transfer}
Yingying Ma, Wei Gong, and Feiyue Mao.
\newblock Transfer learning used to analyze the dynamic evolution of the dust
  aerosol.
\newblock \emph{Journal of Quantitative Spectroscopy and Radiative Transfer},
  153:\penalty0 119--130, 2015.

\bibitem[Namkoong and Duchi(2016)]{namkoong2016stochastic}
Hongseok Namkoong and John~C Duchi.
\newblock Stochastic gradient methods for distributionally robust optimization
  with f-divergences.
\newblock \emph{Advances in neural information processing systems}, 29, 2016.

\bibitem[Ogoe et~al.(2015)Ogoe, Visweswaran, Lu, and
  Gopalakrishnan]{ogoe2015knowledge}
Henry~A Ogoe, Shyam Visweswaran, Xinghua Lu, and Vanathi Gopalakrishnan.
\newblock Knowledge transfer via classification rules using functional mapping
  for integrative modeling of gene expression data.
\newblock \emph{BMC bioinformatics}, 16\penalty0 (1):\penalty0 1--15, 2015.

\bibitem[Pan and Yang(2009)]{pan2009}
Sinno~Jialin Pan and Qiang Yang.
\newblock A survey on transfer learning.
\newblock \emph{IEEE Transactions on knowledge and data engineering},
  22\penalty0 (10):\penalty0 1345--1359, 2009.

\bibitem[Papadopoulos et~al.(2002)Papadopoulos, Proedrou, Vovk, and
  Gammerman]{papadopoulos2002inductive}
Harris Papadopoulos, Kostas Proedrou, Volodya Vovk, and Alex Gammerman.
\newblock Inductive confidence machines for regression.
\newblock In \emph{European Conference on Machine Learning}, pages 345--356.
  Springer, 2002.

\bibitem[Prettenhofer and Stein(2010)]{prettenhofer2010cross}
Peter Prettenhofer and Benno Stein.
\newblock Cross-language text classification using structural correspondence
  learning.
\newblock In \emph{Proceedings of the 48th annual meeting of the association
  for computational linguistics}, pages 1118--1127, 2010.

\bibitem[Sadinle et~al.(2019)Sadinle, Lei, and Wasserman]{sadinle2019least}
Mauricio Sadinle, Jing Lei, and Larry Wasserman.
\newblock Least ambiguous set-valued classifiers with bounded error levels.
\newblock \emph{Journal of the American Statistical Association}, 114\penalty0
  (525):\penalty0 223--234, 2019.

\bibitem[Scott(2019)]{scott2019generalized}
Clayton Scott.
\newblock A generalized neyman-pearson criterion for optimal domain adaptation.
\newblock In \emph{Algorithmic Learning Theory}, pages 738--761. PMLR, 2019.

\bibitem[Shafer and Vovk(2008)]{shafer2008tutorial}
Glenn Shafer and Vladimir Vovk.
\newblock A tutorial on conformal prediction.
\newblock \emph{Journal of Machine Learning Research}, 9\penalty0 (3), 2008.

\bibitem[Shen et~al.(2020)Shen, Liu, and Xie]{shen2020fusion}
Jieli Shen, Regina~Y Liu, and Min-ge Xie.
\newblock i fusion: Individualized fusion learning.
\newblock \emph{Journal of the American Statistical Association}, 115\penalty0
  (531):\penalty0 1251--1267, 2020.

\bibitem[Shimodaira(2000)]{shimodaira2000improving}
Hidetoshi Shimodaira.
\newblock Improving predictive inference under covariate shift by weighting the
  log-likelihood function.
\newblock \emph{Journal of statistical planning and inference}, 90\penalty0
  (2):\penalty0 227--244, 2000.

\bibitem[Sugiyama et~al.(2007)Sugiyama, Nakajima, Kashima, Buenau, and
  Kawanabe]{sugiyama2007direct}
Masashi Sugiyama, Shinichi Nakajima, Hisashi Kashima, Paul Buenau, and Motoaki
  Kawanabe.
\newblock Direct importance estimation with model selection and its application
  to covariate shift adaptation.
\newblock \emph{Advances in neural information processing systems}, 20, 2007.

\bibitem[Swarup and Ray(2006)]{swarup2006cross}
Samarth Swarup and Sylvian~R Ray.
\newblock Cross-domain knowledge transfer using structured representations.
\newblock In \emph{Aaai}, volume~6, pages 506--511, 2006.

\bibitem[Tibshirani et~al.(2019)Tibshirani, Foygel~Barber, Candes, and
  Ramdas]{tibshirani2019conformal}
Ryan~J Tibshirani, Rina Foygel~Barber, Emmanuel Candes, and Aaditya Ramdas.
\newblock Conformal prediction under covariate shift.
\newblock \emph{Advances in neural information processing systems}, 32, 2019.

\bibitem[van~de Geer(2008)]{van2008high}
Sara~A. van~de Geer.
\newblock {High-dimensional generalized linear models and the lasso}.
\newblock \emph{The Annals of Statistics}, 36\penalty0 (2):\penalty0 614 --
  645, 2008.

\bibitem[Van~Rooyen and Williamson(2017)]{van2017theory}
Brendan Van~Rooyen and Robert~C Williamson.
\newblock A theory of learning with corrupted labels.
\newblock \emph{J. Mach. Learn. Res.}, 18\penalty0 (1):\penalty0 8501--8550,
  2017.

\bibitem[Vovk(2013)]{vovk2013transductive}
Vladimir Vovk.
\newblock Transductive conformal predictors.
\newblock In \emph{IFIP International Conference on Artificial Intelligence
  Applications and Innovations}, pages 348--360. Springer, 2013.

\bibitem[Vovk et~al.(2005)Vovk, Gammerman, and Shafer]{vovk2005algorithmic}
Vladimir Vovk, Alexander Gammerman, and Glenn Shafer.
\newblock \emph{Algorithmic learning in a random world}.
\newblock Springer Science \& Business Media, 2005.

\bibitem[Vovk et~al.(2009)Vovk, Nouretdinov, and Gammerman]{vovk2009line}
Vladimir Vovk, Ilia Nouretdinov, and Alex Gammerman.
\newblock On-line predictive linear regression.
\newblock \emph{The Annals of Statistics}, pages 1566--1590, 2009.

\bibitem[Wang and Qiao(2025)]{wang2025conformal}
Baozhen Wang and Xingye Qiao.
\newblock Conformal inference of individual treatment effects using conditional
  density estimates.
\newblock \emph{arXiv preprint arXiv:2501.14933}, 2025.

\bibitem[Wang and Mahadevan(2011)]{wang2011}
Chang Wang and Sridhar Mahadevan.
\newblock Heterogeneous domain adaptation using manifold alignment.
\newblock In \emph{Twenty-second international joint conference on artificial
  intelligence}, 2011.

\bibitem[Weiss et~al.(2016)Weiss, Khoshgoftaar, and Wang]{weiss2016}
Karl Weiss, Taghi~M Khoshgoftaar, and DingDing Wang.
\newblock A survey of transfer learning.
\newblock \emph{Journal of Big data}, 3\penalty0 (1):\penalty0 1--40, 2016.

\bibitem[Zhu et~al.(2011)Zhu, Chen, Lu, Pan, Xue, Yu, and
  Yang]{zhu2011heterogeneous}
Yin Zhu, Yuqiang Chen, Zhongqi Lu, Sinno~Jialin Pan, Gui-Rong Xue, Yong Yu, and
  Qiang Yang.
\newblock Heterogeneous transfer learning for image classification.
\newblock In \emph{Twenty-fifth aaai conference on artificial intelligence},
  2011.

\bibitem[Zou et~al.(2015)Zou, Zhu, Zhu, Baydogan, Wang, and
  Li]{zou2015transfer}
Na~Zou, Yun Zhu, Ji~Zhu, Mustafa Baydogan, Wei Wang, and Jing Li.
\newblock A transfer learning approach for predictive modeling of degenerate
  biological systems.
\newblock \emph{Technometrics}, 57\penalty0 (3):\penalty0 362--373, 2015.

\end{thebibliography}

\appendix

\section{PROOFS}\label{Appendix:proofs}

\subsection{Weighted Quantile Lemma}\label{Appendix:wqL}
\citet{tibshirani2019conformal} defined a generalized notion of exchangeability called weighted exchangeability.

\noindent\textbf{Definition 1 \citep{tibshirani2019conformal}.} We call random variables $V_1, \ldots, V_n$ weighted exchangeable, with weight functions $w_1, \ldots, w_n$, if the density $f$ of their joint distribution can be factorized as
$$
f(v_1, \ldots, v_n) = \prod_{i=1}^{n} w_i(v_i) \cdot g(v_1, \ldots, v_n),
$$
where $g$ does not depend on the ordering of its inputs, i.e., $g(v_{\sigma(1)}, \ldots, v_{\sigma(n)}) = g(v_1, \ldots, v_n)$ for any permutation $\sigma$ of $1, \ldots, n$.

According to Lemma 2 in \citep{tibshirani2019conformal}, under $\CS$ and $\CSPD$, all $Z_i$ are weighted exchangeable, with weight function $w = dQ_X/dP_X$ for the source sample and weight function $w = 1$ for the target sample. Then we can utilize the following Lemma \ref{lemmaTib} to prove Lemma \ref{lemma1}.

\begin{lemma}[\citet{tibshirani2019conformal}]\label{lemmaTib}
Let $Z_i, i = 1, \ldots, n + 1$ be weighted exchangeable, with weight functions $w_1, \ldots, w_{n+1}$. Let $V_i = S(Z_i, Z_{1:(n+1)})$, for $i = 1, \ldots, n + 1$, and $S$ is an arbitrary score function. Define
$$p_i^w (z_{1}, \ldots, z_{n+1}) = \frac{\sum_{\sigma:\sigma(n+1)=i} \prod_{j=1}^{n+1} w_j(z_{\sigma(j)})}{\sum_{\sigma} \prod_{j=1}^{n+1} w_j(z_{\sigma(j)})},$$  $i = 1, \ldots, n + 1,$ where the summations are taken over permutations $\sigma$ of the numbers $1, \ldots, n + 1$. Then for any $\beta \in (0, 1)$,
\begin{align*}
 P\Biggl(V_{n+1} \leq \text{Quantile}\Bigl(&\beta; \sum_{i=1}^{n}  p_i^w (Z_{1}, \ldots, Z_{n+1})\delta_{V_i}\\ &+ p_{n+1}^w(Z_{1}, \ldots, Z_{n+1})\delta_{\infty}\Bigr)\Biggr) \geq \beta.  
\end{align*}

\end{lemma}

\subsection{Proof of Lemma 1}\label{Appendix:Lemma1}
\setcounter{lemma}{0}
\setcounter{theorem}{0}
\begin{lemma}
Obtain the set-valued classifier based on Algorithm \hyperref[alg1]{1}, that is, estimate $\hat{\eta}_{P,j}$ using the first split of the source training data $\mathcal{S}_1$ and classifier $\mathcal{A}$, evaluate $\hat{\eta}_{P,j}(x)$ on $\mathcal{R}^j\cup \{T\}$ (the second split of the class $j$ source training data, the target training data, and the test data point), compute the weights $\tilde w_{ij}$ and $\tilde w_{Tj}$ according to \eqref{weight}, and 
\begin{equation}
    \hat{t}_{j,\alpha}=\text{Quantile} \left(\alpha, \sum_{{i\in \mathcal{R}^j}}\tilde{w}_{ij} \delta_{\hat{\eta}_{P,j}(x_i)}+\tilde{w}_{Tj}\delta_{\infty}\right).
\end{equation}
Then we have,
\[
\mathbb{Q}\left( \hat{\eta}_{P,j}(X_{T}) \geq \hat{t}_{j,\alpha} | Y_{T} = j \right) \geq 1-\alpha,
\]
where $\mathbb{Q}$ is with respect to all the data points in $\mathcal{S}_2\cup \mathcal{T}\cup \{T\}$
\end{lemma}
\begin{proof}
 Since $Q_{X|Y=j}$ is absolutely continuous with respect to $P_{X|Y=j}$, we have $\hat{\eta}_{P,j}(X_i)$ are weighted exchangeable for $i\in \mathcal{R}^j$. Then in Lemma \ref{lemmaTib}\citep{tibshirani2019conformal} and for a class $j$, we let $V_{i,j} = -\hat{\eta}_{P,j}(X_i) \leq 0$ be the score function, then we have 
 \begin{align*}
     \mbbQ\Bigl(-\hat{\eta}_{P,j}(X_T) \leq \text{Quantile} (\beta, \sum_{i\in\mathcal{S}^j\cup\mathcal{T}^j} & \mathds{1}_{(Y_i=j)}\tilde{w}_{ij}\delta_{-\hat{\eta}_{P,j}(X_i)}\\
     &+\tilde{w}_{Tj}\delta_{\infty})\Bigr)\geq \beta,
 \end{align*}
 
 which immediately is equivalent to
 \begin{align*}
     \mbbQ\Bigl(\hat{\eta}_{P,j}(X_T) \geq \text{Quantile} (1 - \beta, \sum_{i\in\mathcal{S}^j\cup\mathcal{T}^j} &\mathds{1}_{(Y_i=j)}\tilde{w}_{ij}\delta_{\hat{\eta}_{P,j}(X_i)}\\&+\tilde{w}_{Tj}\delta_{\infty})\Bigr)\geq \beta.
 \end{align*}
 
 Replacing $\beta$ by $1 - \alpha$, we have Lemma 1 proofed.
\end{proof}

\subsection{Proof of Theorem 1}\label{Appendix:proof1}
\begin{theorem}
Assume for every class $j$, $Q_{X|Y=j}$ is absolutely continuous w.r.t. $P_{X|Y=j}$, and $\CSPD$ assumption holds. For a test point $(X_T, Y_T)$, let $\hat{C}(X_T)$ be the set-valued classifier obtained from Algorithm \hyperref[alg1]{1}. Then $\forall\alpha\in (0,1),$ we have
\[\mathbb{Q}\left(Y_T \in \hat{C}(X_T) \right) \geq 1-\alpha,\]
where $\mathbb{Q}$ is with respect to all the data points in $\mathcal{S}_2\cup \mathcal{T}\cup \{T\}$.
\end{theorem}
\begin{proof}
By construction of Algorithm \hyperref[alg1]{1}, $Y_{T}\in \hat{C}(X_T)$ is equivalent to 
$$\hat{\eta}_{P,j}(X_{T}) \geq \text{Quantile} \left(\alpha, \sum_{{i\in \mathcal{R}^j
   }}\tilde{w}_{ij} \delta_{\hat{\eta}_{P,j}(x_i)}+\tilde{w}_{Tj}\delta_{\infty}\right).$$
   Applying Lemma \ref{lemma1}, we immediately have the result.
\end{proof}
\begin{proposition}
    In general, for each class $j$, if \(\hat{\tilde{w}}(\cdot) \neq \tilde{w}(\cdot)\), define \(\Delta w_j = \frac{1}{2} \mathbb{E}|\hat{\tilde{w}}_{ij}(X) - \tilde{w}_{ij}(X)|\). Assume we have \(\mathbb{E}[\hat{\tilde{w}}_{ij}(X)] < \infty\). In this case, the coverage is lower bounded by \(1 - \alpha - \max_j\Delta w_j\),
 \[
\mathbb{Q}\left( \hat{\eta}_{P,j}(X_{T}) \geq \hat{t}_{j,\alpha} \right) \geq 1-\alpha- \max_j\Delta w_j,
\]
\end{proposition}
\begin{proof}
    This proposition can be directly proved by using the Theorem 3 in \cite{lei2021conformal}, which yields for each class $j$,
    \[
\mathbb{Q}\left( \hat{\eta}_{P,j}(X_{T}) \geq \hat{t}_{j,\alpha}|Y_T = j \right) \geq 1-\alpha- \Delta w_j.
\]
\end{proof}

\begin{figure*}[t]
    \centering
    \includegraphics[width = \textwidth]{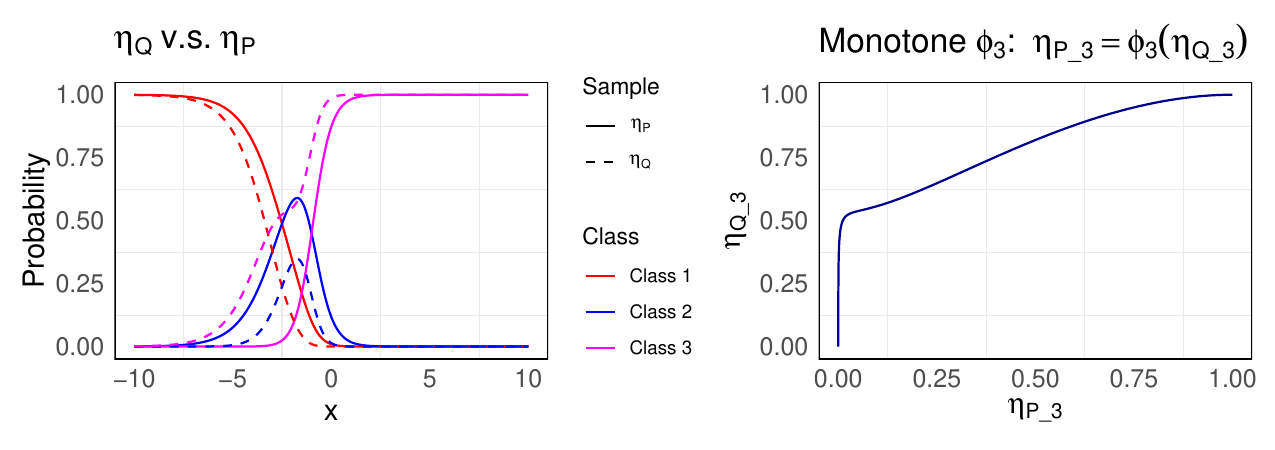}
    \caption{An illustration of the simulation study when $r = 2$.}
    \label{illu}
\end{figure*}

\subsection{Proof of Theorem 2}\label{Appendix:proof2}
\begin{theorem}
Define $m = |\mathcal{S}_1|$ as the size of data set used to estimate $\hat{\eta}_{P,j}$, and $n_{j} = |\mathcal{R}^j|$ as the size of calibration set for class $j$. Under $\gCSPD$ assumption at $\alpha$ and assumptions \textbf{(A)}, \textbf{(B)}, if $Q_{X|Y=j}$ is absolutely continuous with respect to $P_{X|Y=j}$ and $w^j$ are bounded, then for each $r>0$, there exists a positive constant $c$ such that for $m$ and $n_{j}$ large enough, with probability at least $1-\theta_{m}-n_{j}^{-r}$
\[\mathbb{Q}_j(\hat{C}_{j,\alpha}\triangle C_{j,\alpha})\leq c\left\{ \delta_{m}^\lambda  + \left( \frac{\log n_{j}}{n_{j}}\right)^{\frac{1}{2}}\right\}.\]
\end{theorem}
\begin{proof}
Denote $$\hat{G}_j(t) = \sum_{i\in (\mathcal{S}2\cap \mathcal{I}^j)\cup (\mathcal{T}\cap \mathcal{I}^j)\cup {T}} \tilde{w}_{ij}\mathds{1}\left\{\eta_{P,j}(x_i) \leq t\right\}$$ the weighted empirical distribution function. Let $\hat{\mathbb{Q}}_j$ be the probability measure corresponding to $\hat{G}_j$. Consider the following event 
\begin{align*}
    E_r=\Biggl\{&\|\eta_{P,j}-\eta_{P,j}\|_\infty\leq\delta_m,\\
    &\sup_t |G_j(t)-\hat{G}_j(t)|\leq c_r(\log n_j/n_j)^{\frac{1}{2}} \Biggr\},
\end{align*}
which has probability at least $1-\theta_m-n_j^{-r}$ for constant $c_r$ depending on $r$.
To see this, the first inequality in event $E_r$ is given by assumption \textbf{(A)}. 

For the second inequality, since $\tilde{w}_{ij} > 0$ are bounded, let's assume $a < \tilde{w}_{ij} < b$. Then we apply this bound
$$\frac{(\sum_i \tilde{w}_{ij})^2}{(\sum_i \tilde{w}_{ij}^2)} = \frac{\|\tilde{w}_{ij}\|_1^2}{\|\tilde{w}_{ij}\|_2^2} \geq \frac{\|\tilde{w}_{ij}\|_1}{\|\tilde{w}_{ij}\|_\infty} > \frac{n_j a}{b}$$
on the right hand side of the weighted empirical distribution inequality (Proposition 3.1 in \citet{chen2019nearest}), for $n_j$ large enough, we have $\mathbb{Q}_j(\sup_t |G_j(t)-\hat{G}_j(t)| > c_r(\log n_j/n_j)^{\frac{1}{2}}) \leq \frac{6}{\log n_j} n_j^{-\frac{2a}{9b}c_r^2+\frac{1}{2}}\leq n_j^{-r}$.

Define $L_j(t)=\left\{ x: \eta_{P,j}(x)\leq t\right\}$ and $\hat{L}_j(t)=\left\{ x: \hat{\eta}_{P,j}(x)\leq t\right\}$.
Let $t_{j,\alpha} =G_j^{-1}(\alpha)$ be the ideal cut-off value for $\eta_{P,j}$. If $t= t_{j,\alpha} - \delta_m - \left[2c_rb_1^{-1}\sqrt{\frac{\log n_j}{n_j}}\right]^{\frac{1}{\lambda}}$, then we have
\begin{align*}
    \hat{\mathbb{Q}}_j\left[\hat{L}_j(t)\right] & \leq \hat{\mathbb{Q}}_j\left[\hat{L}_j(t+\delta_m)\right]\\
    &=\hat{G}_j(t+\delta_m)\\
    &\leq G_j(t+\delta_m)+c_r\sqrt{\frac{\log n_j}{n_j}} \\
    &\leq G_j \Bigl\{ t_{j,\alpha}- \left[2c_rb_1^{-1}\sqrt{\frac{\log n_j}{n_j}}\right]^{\frac{1}{\lambda}} \Bigr\} + c_r\sqrt{\frac{\log n_j}{n_j}}\\
    &\leq G_j(t_{j,\alpha})-c_r\sqrt{\frac{\log n_j}{n_j}}\\
    &=1 - \alpha - c_r\sqrt{\frac{\log n_j}{n_j}}\\
    &<1 - \alpha - n_j^{-1}\\
    &\leq \hat{\mathbb{Q}}_j\left[\hat{L}_j(\hat{t}_{j,\alpha})\right] 
\end{align*}
Therefore 
\begin{equation}\label{thmeq1}
    \hat{t}_{j,\alpha} \geq t_{j,\alpha} - \delta_m - \left
(2c_rb_1^{-1}\sqrt{\frac{\log n_j}{n_j}}\right)^{\frac{1}{\lambda}} 
\end{equation}
Similarly, we can show the reverse inequality 
\begin{equation}\label{thmeq2}
    \hat{t}_{j,\alpha} \leq  t_{j,\alpha} + \delta_m + \left(2c_rb_1^{-1}\sqrt{\frac{\log n_j}{n_j}}\right)^{\frac{1}{\lambda}}
\end{equation}
Combining \eqref{thmeq1} and \eqref{thmeq2}, we have $|\hat{t}_{j,\alpha} - t_{j,\alpha}| \leq \delta_m + \left(2c_rb_1^{-1}\sqrt{\frac{\log n_j}{n_j}}\right)^{\frac{1}{\lambda}}.$
Now on event $E_r$, we have
\begin{align*}
    & \mathbb{Q}_j(\hat{C}_{j,\alpha}\triangle C_{j,\alpha}) \\
    & = \mathbb{Q}_j(\hat{\eta}_{P,j}(X)\geq \hat{t}_{j,\alpha},\eta_{P,j}(X)<t_{j,\alpha})\\
    &\leq \mathbb{Q}_j \Bigl\{ t_{j,\alpha} - 2\delta_m - \Bigl[2c_rb_1^{-1}\sqrt{\frac{\log n_j}{n_j}}\Bigr]^{\frac{1}{\lambda}} < \eta_{P,j}(X) < t_{j,\alpha} \Bigr\}\\
    &= G_j(t_{j,\alpha}) - G_j\Bigl\{ t_{j,\alpha} - 2\delta_m - \Bigl[2c_rb_1^{-1}\sqrt{\frac{\log n_j}{n_j}}\Bigr]^{\frac{1}{\lambda}} \Bigr\}\\
    &\leq b_2 \Bigl\{2\delta_m + \Bigl[2c_rb_1^{-1}\sqrt{\frac{\log n_j}{n_j}}\Bigr]^{\frac{1}{\lambda}}\Bigr\}^\lambda\\
    &\leq c \Bigl(\delta_m^\lambda + \sqrt{\frac{\log n_j}{n_j}}\Bigr)
\end{align*}
Here $c$ is some constant depending on $c_r, b_1, b_2, \lambda$. The next-to-last inequality follows from assumption \textbf{(B)} and holds when $m$ and $n$ are large enough so that $2\delta_m + \left[2c_rb_1^{-1}\sqrt{\frac{\log n_j}{n_j}}\right]^{\frac{1}{\lambda}} \leq d_0$

\end{proof}

\begin{figure*}[t]
    \centering
    \includegraphics[width = \textwidth]{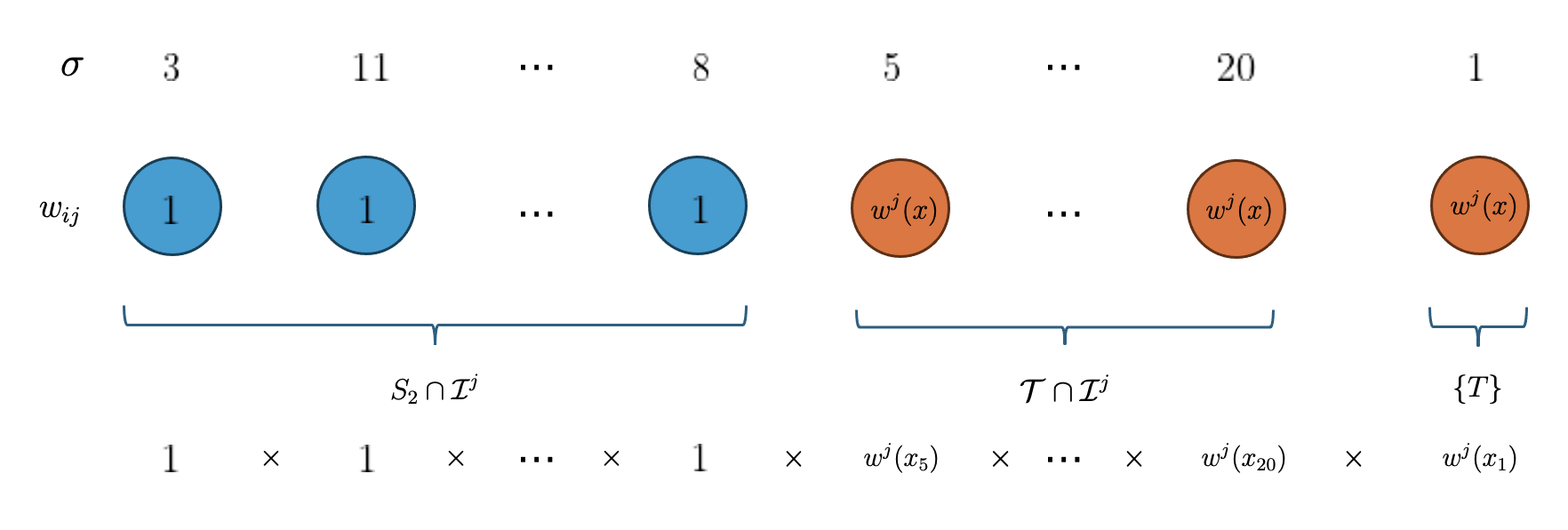}
    \caption{An illustration for one term in the weight calculation under a specific permutation}
    \label{weight_illu}
\end{figure*}
\section{NUMERICAL STUDIES}
\subsection{Overview}
All experiments were conducted on a desktop with an Intel i7-8700 CPU and 16 GB of RAM, using the R parallel computing environment without GPU support. The simulation study involved 1000 replications, with each replication cycling through $r = 1, 1.2, 1.4, 1.6, 1.8,$ and $2.0$, and required approximately 7 hours to complete. For the semi-synthetic experiment can be completed in less than 15 minutes.

For efficiency, in Algorithm \hyperref[alg1]{1},
we compute the cutoff $\hat{t}_{j,\alpha}$ as $\text{Quantile} \left(\alpha, \sum_{{i\in \mathcal{R}^j}}\tilde{w}_{ij} \delta_{\hat{\eta}_{P,j}(x_i)}\right).$ Comparing to the original formula \eqref{threshold}, we omit the last term which place point mass on $\infty$. The risk of removing such term can be negligible in practice. And that means, for every test point instance, we no longer need to update the weights in \ref{weight}, which greatly alleviate the computational burden.

\subsection{Illustration for the Simulation Study}

\textbf{Relabeling process in the simulation study.} For classes 1 and 2, we set $\phi^{-1}(t) = r^t$, where 
$\eta_{Q,j}(x) = \phi^{-1}(\eta_{P,j}(x))$. For class 3, the posterior probability $\eta_{Q,3} (x) = 1 - \eta_{Q,2}(x) - \eta_{Q,1}(x)$. We can verify that there exist an increasing function $\phi_3$ for class 3 (see the right panel of Figure \ref{illu}) that satisfies $\eta_{P,3}(x)=\phi_3(\eta_{Q,3}(x)).$ Then we sample the label based on $\eta_{Q,j}(x)$

\begin{remark}
 One nice property of binary classification is that if $\CSPD$ holds for class 1, then it automatically holds for class 2. This is because once there exists a strictly increasing function $\phi_1$ such that $\eta_{P,1}(x)=\phi_1(\eta_{Q,1}(x)),$ we would immediately have another strictly increasing function $\phi_2(t) = 1 - \phi_1(1-t)$ such that $\phi_2(\eta_{Q,2}(x)) = 1 - \phi_1(1 -\eta_{Q,2}(x))=\eta_{P,2}(x).$ This nice property no longer holds in the multicategory setting in general. Nonetheless, there are scenarios where $\CSPD$ holds for all classes.  
\end{remark}

\subsection{Supplementary Details of the Maternal Health Risk data}
The Maternal Health Risk data \citep{misc_maternal_health_risk_863}, collected from various health facilities in rural Bangladesh through an IoT-based system, comprises $N = 1013$ instances. Each instance includes a class label $Y$ (risk intensity level during pregnancy) and a 6-dimensional covariate $X$ (age, systolic blood pressure, diastolic blood pressure, blood glucose levels, body temperature, and heart rate).

We repeat the experiment 1000 times. For each time, we split the data into $D_{\text{Source}}$, $D_{\text{Target Train}}$, and $D_{\text{Target Test}}$
with a $5:1:4$ ratio. To simulate $\gCSPD$, we start by training an XGBoost model on the entire dataset to robustly estimate the posterior probabilities $\hat{\eta}_{P,j}$. A transformation function $\phi^{-1}_j(x) = x^{1.2}$ is then applied to the posterior probabilities $\hat{\eta}_{Q,j}$ for the first two classes $j = 1,2$. The remaining  procedure is the same as in the simulation study, that is, we set $\hat{\eta}_{Q,3} = 1- \hat{\eta}_{Q,1} - \hat{\eta}_{Q,2}$. And draw labels for the target sample, based on $\hat{\eta}_{Q,j}$. This procedure ensures the shift in all classes satisfies $\gCSPD$ for most of the $\alpha$ levels (See figure \ref{fig5} (a)).

\subsection{Weight Estimation}
Here, we describe the estimation of the covariate likelihood ratio $w^j = dQ_{X|Y=j}/dP_{X|Y=j}$, following the strategy used in \cite{tibshirani2019conformal}. For each class $j$, we augment the covariates to the feature-class pairs $(X_i, C_i)$, where $C_i = 0$ for the target training sample and $C_i = 1$ for the source training sample. Then we train an XGBoost model on the augmented data, to obtain the estimated probabilities $\mathbb{P}(C=i|X=x)$. Noting that\[\frac{\mathbb{P}(C=1|X=x)}{\mathbb{P}(C=0|X=x)} = \frac{\mathbb{P}(C=1)}{\mathbb{P}(C=0)}\cdot\frac{dQ_{X|Y=j}}{dP_{X|Y=j}}(x),\]
we can take the left-hand side as a estimate of initial weight function $w^j(x)$ (since it is proportional to the covariate likelihood ratio).

\subsection{A Toy Example of Weight Calculation}
As a toy example, consider the permutation $\sigma = (3,11,\dots,8,5,\dots,20,1)$ which satisfies $\sigma(T) = i =  1$. The first $N_j^S = |(\mathcal{S}_2\cap \mathcal{I}^j)|$ entries of the permutation involve data points $x_3, x_{11},\dots, x_8$, which are evaluated using the function $w_{ij}(x) = 1$, and the remaining $N_j^T+1 = |(\mathcal{T}\cap \mathcal{I}^j)\cup \{T\}|$ entries involve data points $x_5, x_{20},\dots, x_{1}$, which are evaluated with the $w_{ij}(x) = w^j(x)$ function. The product of $w(x_5), w(x_{20}),\dots, w(x_1)$ then forms one term in the numerator of \eqref{weight} (see Figure \ref{weight_illu} for an illustration). 

\section{MULTI-SOURCE}\label{subsec:methodMS}

In this section, we discuss an extension of our method when multiple source samples are available. First of all, we may no longer set the initial weight for those instances from the source samples to $1$, as was done before; see the definition of the initial weight function before. Instead, we set the initial weight function to be 1 for one reference source sample (typically we recommend the largest source sample, indexed as the first source sample $P_1$). Next, the initial weight functions for the other source samples are set to be $dP_{k,x|y=j}/dP_{1,x|y=j}$, defined as the covariate likelihood ratio of the $k$th source distribution over the reference source distribution, conditional on class $j$. Lastly, initial weight functions for the target data are set to $dQ_{x|y=j}/dP_{1,x|y=j}$. With all the initial weight functions set, we can define the weight for each training data instance using a formula similar to (\ref{weight}). However, its definition can no longer be cast as elementary symmetric polynomials, and hence we may no longer use the Newton's Identity to simplify the computation.

There exist a couple of ways to bypass the computational challenge. One approach is to consider a mixture distribution of all source samples, defined as $P_M = \sum r_k P_k$, where $r_k$ is the proportion of $k$-th source sample among all the source samples. In this case, we proceed with our original method by treating the source data as one single distribution. 

Another approach is one of divide and conquer. We may apply Algorithm \hyperref[alg1]{1} separately for each source sample. Using the $l$-th source, we obtain $\hat{C}_\ell(x) = \{ j : \hat{\eta}_{p,j}^{(\ell)} (x) > \hat{\eta}_{j,\alpha}^{(\ell)} \}$. One intuitive way to combine these prediction sets $\hat{C}_\ell(x)'s$ is to find their Steiner centroid, $\hat{S}$, which minimizes the total symmetric difference to all these sets. Mathematically, we are looking for $\hat{S} = \arg\min_S \sum_{\ell=1}^m |S \triangle \hat{C}_\ell(x)|$. Alternatively, inspired by fusion learning \citep{shen2020fusion}, let $F_\ell$ be the empirical (weighted) cumulative distribution function of $\hat{\eta}_{p,j}^{(\ell)} (X_i)$ where $X_i$'s are class $j$ observations from the target training sample and the $\ell$-th ranking source samples. The fact that $\hat{C}_\ell(x) = \{j : F_\ell(\hat{\eta}_{p,j}^{(\ell)} (x)) > \alpha \}$ motivates to define the combined prediction set as $\{j : \frac{1}{M} \sum_{\ell=1}^M F_\ell(\hat{\eta}_{p,j}^{(\ell)} (\mathbf{x})) > \alpha \}$, assuming that all the sources are equally transferable to the target.

We choose not to pursue any of these extensions in this article. Computational refinement and theoretical studies of these extensions will be left as future research directions.


\end{document}